\documentclass[letterpaper, 10 pt, conference]{ieeeconf}  

\IEEEoverridecommandlockouts                              

\let\labelindent\relax

\usepackage{graphics} 
\usepackage{epsfig} 
\usepackage{mathptmx} 
\usepackage{times} 
\usepackage{amsmath} 
\usepackage{amssymb}  
\usepackage{mathrsfs}
\usepackage[figuresright]{rotating}
\usepackage{algorithm}
\usepackage{enumitem}
\usepackage{algpseudocode}
\usepackage{caption}
\usepackage{subcaption}

\newtheorem{theorem}{Theorem}

\newtheorem{lemma}{Lemma}

\usepackage{xcolor}
\usepackage{soul}
\usepackage{siunitx}

\usepackage{ifluatex}
\ifluatex
  \usepackage{pdftexcmds}
  \makeatletter
  \let\pdfstrcmp\pdf@strcmp
  \let\pdffilemoddate\pdf@filemoddate
  \makeatother
\fi

\usepackage{svg}

\title{\LARGE \bf
Scenario Convex Programs for Dexterous Manipulation\\ under Modeling Uncertainties
}

\author{Berk Altiner$^{1,2}$, Adnane Saoud$^{3,1}$, Alex Caldas$^{2}$, Maria Makarov$^{1}$
\thanks{$^{1}$Laboratoire des signaux et sysèmes, Université Paris-Saclay, CNRS, CentraleSupélec, 91190, Gif-sur-Yvette, France
        {\tt\small maria.makarov@centralesupelec.fr }}%
\thanks{$^{2}$ESME Research Lab., ESME Sudria Engineering School, Ivry-sur-Seine 94200, France
        {\tt\small alex.caldas@esme.fr}}
        \thanks{$^{3}$College of Computing,  University Mohammed VI Polytechnic,UM6P, Benguerir, Morocco,
{\tt\small adnane.saoud@um6p.ma}}%
}

\begin{document}

\maketitle
\thispagestyle{empty}
\pagestyle{empty}

\begin{abstract}
This paper proposes a new framework to design a controller for the dexterous manipulation of an object by a multi-fingered hand. To achieve a robust manipulation and wide range of operations, the uncertainties on the location of the contact point and multiple operating points are taken into account in the control design by sampling the state space. The proposed control strategy is based on a robust pole placement using LMIs. Moreover, to handle uncertainties and different operating points, we recast our problem as a robust convex program (RCP). We then consider the original RCP as a scenario convex program (SCP) and solve the SCP by sampling the uncertain grasp map parameter and operating points in the state space. For a required probabilistic level of confidence, we quantify the feasibility of the SCP solution based on the number of sampling points. The control strategy is tested in simulation in a case study with contact location error and different initial grasps.
\end{abstract}

\section{INTRODUCTION}
Dexterous manipulation involves manipulating, replacing, and reorienting objects using a gripper or multi-fingered hand in the hand space. 
This paper specifically addresses the design of object motion control to guarantee performance under uncertainties in the hand-object model. We focus on uncertainties related to the object's geometry or pose within the hand, which are common and can lead to tracking errors or loss of object stability. The proposed control design method can be complemented with various internal force generation strategies. 

Several adaptive and robust approaches were previously proposed to deal with uncertainties in the dynamic model of the hand-object system (e.g. recent works \cite{khadivar2023adaptive,shaw2019robust} and references therein). 
An important class of control strategies in dexterous manipulation employs object-level impedance control, where stiffness and damping control parameters are to be selected to achieve the desired object-level impedance behavior. In \cite{wimboeck2006passivity}, a passivity-based control design is proposed, along with the virtual grasp map concept which circumvents the need for precise knowledge of real contact points. Moving beyond static grasp configurations, \cite{pfanne2020object} suggests a combination of motion control, internal force control based on optimization, and null-space torque control, although it does not delve into controller gain tuning. Impedance learning from human demonstration is explored in \cite{li2014learning}. Task-oriented end-point stiffness selection for a variable-stiffness hand is proposed in  \cite{martin2023task}. 

A model-based robust control strategy has been proposed recently in \cite{caldas2015object}, and tested in \cite{caldas2022task} on an experimental setup, to explicitly address uncertainties in contact point locations and system geometry. To achieve robust object motion tracking in the presence of contact uncertainties, Linear Matrix Inequality (LMI) formulations of regional pole placement \cite{chilali1996h} are exploited within an object-level state-feedback control structure, which can be closely related to impedance control schemes.
%
%
The present article aims to overcome some limitations of these previous results. In \cite{caldas2015object} and \cite{caldas2022task}, the controller design is based on the discretization of the uncertainty set along a grid. While practical performance is maintained with fine discretization, there is no theoretical guarantee for other points in the uncertainty set. In contrast, the LPV approach from \cite{caldas2019lpv} enables controller design for a polytopic system description, ensuring motion tracking for all systems within the polytope. However, this method may be conservative. Both approaches face constraints due to the linearization of the uncertain dynamic model of the hand-object system, limiting performance guarantees to operation near the considered equilibrium point. 

Building on these previous works, we focus in this article on object-level motion control design and propose a methodology based on the scenario approach \cite{calafiore2006scenario} to achieve guaranteed performance with a probabilistic level of confidence by design. This approach efficiently tackles optimization problems with uncertain data, offering theoretical guarantees on the resulting controller and broadening the closed-loop system's operational range. The main benefit of this approach is that it solves an uncertain convex optimization problem via sampling constraints by providing explicit probabilistic bounds on the solution for a given number of samples. In the literature, this problem is investigated from two different perspectives: the feasibility of the solution and the optimal objective value. For the feasibility of the solution, following the feasibility results in \cite{calafiore2006scenario}, a bound on the sample complexity that holds for all convex problems is presented in \cite{campi2008exact}. The probabilistic relation between the optimal values of the Scenario Convex Program (SCP) and the Robust Convex Program (RCP) is developed in \cite{esfahani2014performance} based on the result in \cite{campi2008exact} and the definition of the worst-case violation probability \cite{kanamori2012worst}. Motivated by these results, we reframe the robust pole placement problem for dexterous manipulation of multi-fingered hands as an RCP to overcome limitations that occur from linear models. This involves constructing a set of uncertain linear systems to account for contact uncertainties and multiple operating points in the state space. Then, the uncertain LMI optimization problem is solved via the scenario approach by exploiting the results in \cite{campi2008exact}, \cite{esfahani2014performance} to ensure the following objectives: motion tracking, robustness against contact uncertainties, and a wide range of operation. 

\subsection{Notation}
We denote the set of real numbers by $\mathbb{R}$. The set of $n$-dimensional vectors, the set of $m\times n$ dimensional matrices and the set of $n\times n$ symmetric matrices are represented by $\mathbb{R}^n$, $\mathbb{R}^{m\times n}$, $\mathbb{S}^n$. The symbols $\mathbb{O}$ and $\mathbb{I}$ represent the zero matrix and the identity matrix of appropriate dimensions. The spectral norm of a matrix is denoted by $||. ||$.  Pseudo-inverse and transpose of a matrix are denoted by $(.)^\dagger$  and $(.)^T$ respectively. Finally, $\mathbb{P}$ represents a given probability measure on the uncertain set. The symmetric matrix $\begin{bmatrix}
    A & \ast\\
    B & C
\end{bmatrix}$ stands for $\begin{bmatrix}
    A & B^{\top}\\
    B & C
\end{bmatrix}$.


\section{State-space model of the hand-object system}
\subsection{The hand-object system dynamics at the object level}
\label{sec:NLmodel}

The multi-fingered hand model is obtained by combining the dynamics of the fingers and the object \cite{caldas2022task}:
\begin{align}
   \underbrace{\begin{bmatrix}
       \dot{x}_o\\
       \ddot{x}_o
   \end{bmatrix}}_{\dot{x}}= \underbrace{\begin{bmatrix}
        \dot{x}_o\\
       -M^{-1}C\dot{x}_o-M^{-1}N
   \end{bmatrix}}_{\kappa(q,x_o,\dot{q},\dot{x}_o)}+\underbrace{\begin{bmatrix}
        \mathbb{O}\\
       -M^{-1}GJ^{-T}_h
   \end{bmatrix}\tau}_{\chi (q,x_o,\dot{q},\dot{x}_o,\tau)}
   \label{eq:NLss}
\end{align}

\noindent where  $q \in \mathbb{R}^{n_q}$ is the vector of the joint angles of the multifingered hand, $x_o\in \mathbb{R}^{n_o}$ is the local coordinate of the object, $x = \begin{bmatrix} x_o & \dot{x}_o \end{bmatrix}^T$ is the state space vector, $M{(q,x_o)} \in \mathbb{R}^{n_o\times n_o}$, $C(q, x_o, \dot{q}, \dot{x}_o) \in \mathbb{R}^{n_o\times n_o}$, $N(q, x_o) \in \mathbb{R}^{n_o}$ are respectively the inertia matrix, the Coriolis and centrifugal matrix and the gravity vector of the hand/object system expressed at the object level, $G \in \mathbb{R}^{n_o\times n_c}$ is the Grasp Map, $J_h(q,x_o) \in \mathbb{R}^{n_c\times n_q}$ is the Jacobian matrix of the hand (with $n_c$ the dimension of the contact frame), $\tau \in \mathbb{R}^{n_q}$ the joint torque vector. Note that in this formulation, the variables $\dot{x}_o$ and $\dot{q}$ are related by:
\begin{align}
\label{eq:contact_constraint}
    J_h \dot{q}=G^T \dot{x}_o
\end{align}

In the considered control structure, the joint control torques are obtained from the object-level control inputs $u$ and $\lambda$:
\begin{equation}
\label{eq:NLtransform}
\tau=\underbrace{{J}_{h}^T\hat{G}^{\dagger}u}_{\tau_{motion}}+\underbrace{{J}_{h}^T\hat{N}_G\lambda}_{\tau_{internal}}
\end{equation}
with $\hat{G}$ the estimate of the Grasp Map and the matrix $\hat{N}_G$ a basis of the kernel of $\hat{G}$. Thus, the first term $\tau_{motion}$ generates the object motion and the second term $\tau_{internal}$ generates internal forces that do not produce motion in the nominal case, but fulfill the contact stability constraints. This paper does not focus on the internal force design and uses the approach exposed in \cite{caldas2015object}.

\subsection{Linearized model for control design}
\label{sec:model}

Our control objectives are: [O1] For a given initial object pose, reach the desired final pose with specified performance criteria (response time, damping, static error), and [O2] perform the manipulation in the presence of contact uncertainties. 

Additional assumptions about the system are listed below, with possible relaxations:
\begin{enumerate}[label=A.\arabic*]
\item The system is not redundant, i.e., there is no internal movement of the fingers for a fixed position of the object, and the grasp is manipulable, i.e. the desired motion can be generated by the fingers. In this case, the hand Jacobian $J_h$ is square and invertible. This assumption could be relaxed with a model taking into account the joint redundancy \cite{murray1994mathematical}.
\item The contact points are fixed, which implies that $G$ is constant. This assumption will be relaxed later by taking into consideration the uncertainties on the contact point location.
\item The influence of the gravity terms $N_h(q,\dot{q})$ and $N_o(x_o)$ is negligible or compensated by the control law. 
\item The estimation of the hand Jacobian is perfect, i.e., $\hat{J}_{h_{eq}}=J_{h_{eq}}$. 
\item  To avoid slipping or rolling, the contact forces remain in friction cones \cite{caldas2019lpv}. This condition is satisfied by the control input $\lambda$ in (\ref{eq:NLtransform}).
\end{enumerate}

We use the following notations for an equilibrium point and an operating point of (\ref{eq:NLss}). An equilibrium $(x_{eq},\tau_{eq})$ is such that $\dot{x}_{eq}=0$, with ${x}_{o_{eq}}$ the object pose and $q_{eq}$ the joint angles. An operating point  $(x_{*},\tau_{*})$ with $x_{*}=\begin{bmatrix} x^T_{o_{*}} & \dot{x}^T_{o_{*}} \end{bmatrix}^T$ is considered in the following for the linearization of (\ref{eq:NLss}). Note that from constraint (\ref{eq:contact_constraint}), the operating variables $q_{*}$ and $\dot{q}_{*}$ can also be defined.

By neglecting higher order terms, the linearization of (\ref{eq:NLss}) around
an operating point yields the affine model 
\begin{align}
    \dot{x}\approx \kappa (x_{*}) + A_{*}(x-x_{*}) + \chi(x_{*},\tau_{*}) + B_{*}(\tau-\tau_{*})
\end{align}
where $A_{*}=\frac{\partial \kappa}{\partial x}\Big |_{x_{*}}+\frac{\partial \chi}{\partial x}\Big |_{(x_{*},\tau_*)}$, $B_{*}=\frac{\partial \chi}{\partial \tau}\Big |_{(x_{*},\tau_*)}$. 

Under the transformation $\tilde{x}=\begin{bmatrix}\tilde{x}_o^T & \dot{\tilde{x}}_o^T\end{bmatrix}^T=x-x_{eq}$ and $\tilde{\tau}=\tau-\tau_{eq}$, one gets
\begin{align}
   \dot{\tilde{x}} &\approx \kappa(x_{*}) + A_{*}(x-x_{*}) +\chi(x_{*},\tau_{*}) + B_{*}(\tau-\tau_{*}) \\
   &\qquad\qquad + A_{*}x_{{eq}}-A_{*} x_{{eq}}+B_{*}\tau_{eq}-B_{*}\tau_{eq}\\
    &=A_{*}\tilde{x}+B_{*}\tilde{\tau}+A_{*}({x}_{{eq}}-x_{*})+B_{*}(\tau_{eq}-\tau_{*}) \\
    &\qquad\qquad +\kappa(x_{*})+\chi(x_{*},\tau_{*})\\
    &=A_{*}\tilde{x} + B_{*}\tilde{\tau} + V_{eq,*}\label{eq:linerized}.
\end{align}
%
%
\noindent with $V_{eq,*}$ including second-order non linear terms.

\begin{figure*}[!h]
\rule{\textwidth}{0.4pt}
\begin{align}
\label{eq:ss2}
\begin{bmatrix}
\dot{\tilde{x}}_o\\
\ddot{\tilde{x}}_o
\end{bmatrix}&=\underbrace{\begin{bmatrix}
        \mathbb{O}_{n_0\times n_0}&\mathbb{I}_{n_0\times n_0}\\
        -M'_*C_*\dot{x}_{o_*}-M^{-1}_*C_*'\dot{x}_{o_*}-(M'_*GJ^{-T}_{h_*}+M^{-1}_*GJ'_{h_*})\tau_*&-M^{-1}_*C''_*\dot{x}_{o_*}-M^{-1}_*C_*
    \end{bmatrix}}_{A_*}\begin{bmatrix}
\tilde{x}_o\\
\dot{\tilde{x}}_o
\end{bmatrix}+\underbrace{\begin{bmatrix}
\mathbb{O}_{n_0\times n_0}\\
M^{-1}_{*}G
\end{bmatrix}}_{B_*}\hat{G}^{\dagger}u+V_{eq,*}\\
V_{eq,*}&=\begin{bmatrix}
    \dot{x}_{o_{*}}\\
    M^{-1}_{*}C_*\dot{x}_{o_{*}}
\end{bmatrix}-\begin{bmatrix}
        \mathbb{O}_{n_0\times n_0}&\mathbb{I}_{n_0\times n_0}\\
        -M'_*C_*\dot{x}_{o_*}-M^{-1}_*C_*'\dot{x}_{o_*}-(M'_*GJ^{-T}_{h_*}+M^{-1}_*GJ'_{h_*})\tau_*&-M^{-1}_*C''_*\dot{x}_{o_*}
    \end{bmatrix}\begin{bmatrix}
{x}_{o_{eq}}-{x}_{o_{*}}\\
\dot{{x}}_{o_{eq}}-\dot{{x}}_{o_{*}}
\end{bmatrix}\notag
\end{align}
\rule{\textwidth}{0.4pt}
\end{figure*}
Taking into account that the joint control torques are expressed using the estimate $\hat{J}_h$ of the hand Jacobian at the operating point instead of $J_h$ in (\ref{eq:NLtransform}), equation (\ref{eq:linerized}) in explicit form with the control input expressed at the object level is transformed to {equation (\ref{eq:ss2})} where, $C'=\frac{\partial C}{\partial x_o}$, $C''=\frac{\partial C}{\partial \dot{x}_{o}}$, $N'=\frac{\partial C}{\partial x_o}$, $N''=\frac{\partial N}{\partial \dot{x}_{o}}$, $M'=\frac{\partial M^{-1}}{\partial x_o}$,  $J_h'=\frac{\partial J_h^{-T}}{\partial x_o}$, and $M_{*}$, $C_{*}$ and $N_{*}$ are defined around the operating point.

\subsection{Model of the Contact Uncertainties}
Geometric uncertainties in object shape and dimensions, as well as on the locations of the contact points, impact the grasp map $G$, which can be expressed as
\begin{equation}
G=G(\delta)
\end{equation}
where $\delta \in \Delta $ is the vector of uncertainties and $\Delta$ is the set of uncertain parameters \cite{caldas2015object}.
Since the grasp map $G$ depends on the uncertain parameters, the set of all possible state-space representations depending on the uncertainties is defined:
\begin{equation}
\label{eq:grasp_uncertain}
\Sigma_{\delta} := \{ \forall \delta \in \Delta \;|\; \dot{x}=A_*(\delta)x+B_*(\delta)u+V_{eq,*}\}
\end{equation}
 The state-space model (\ref{eq:ss2}) with assumptions A.1-A.6 consists of the linearized $M_{*}$ and $C_{*}$ matrices, which depend on operating points ($x_{*}$, $\tau_*$). Then, for $(x_{*}, \tau_*, \delta)\in  X \times U \times \Delta$, $\xi \triangleq (x_{{eq}}, \tau_{eq}, \delta)$ and $\Xi = X \times U \times \Delta$, the set of all possible state-space representations depending on the operating point $(x_{*}, \tau_*)$ and the uncertainty $\delta$ can be augmented as
\begin{equation}
\label{eqn:unc_sys}
\Sigma_{\xi} := \{ \dot{x}=\tilde{A}(\xi)x+\tilde{B}(\xi)u |\; \xi \in \Xi \}
\end{equation}
where \begin{equation}
\label{eqn:dyn}
\tilde{A}(\xi)=\begin{bmatrix}
    A_*(\xi)&V(\xi)\\0&0
\end{bmatrix},\quad \tilde{B}(\xi)=\begin{bmatrix}
    B_*(\xi)\\0
\end{bmatrix}.
\end{equation}

\section{{Scenario Convex Programming}}
\subsection{{Formulation of the Scenario Convex Program}}
Under static state-feedback of gain $L_c$, O1 can be satisfied by setting the closed-loop eigenvalues with a control gain to a specific region in the left-half plane. It is called $\mathscr{D}$-region and it can be described by LMI constraints \cite{chilali1996h}. The state feedback is in the form of
\begin{equation}
\label{eq:ctrl}
u(x)=e-L_cx
\end{equation}
where $e$ is the input of the closed-loop system and $L_c$ is the state feedback matrix obtained from the solution of an optimization problem with LMIs constraints. If we aim to find a feasible solution for the regional pole placement problem for the set of uncertain systems $\Sigma_{\xi}$, and construct the controller based on the feasible solution, the optimization problem can be expressed as follows:

\begin{small}
\begin{align}
\label{eq:opt_problem2}
& \underset{P, \; Y , \; \gamma}{\text{min}}
& &  \gamma \\  
& \text{s.t.}
& & 
(\tilde{A}(\xi)P-\tilde{B}(\xi)Y)+(P\tilde{A}^T(\xi)-Y^T\tilde{B}^T(\xi))+2\alpha P \prec \gamma I \label{eq:opt_problem2_1} \\ 
& 
& & \begin{bmatrix}
-rP&\tilde{A}(\xi)P-\tilde{B}(\xi)Y\\
P\tilde{A}^T(\xi)-Y^T\tilde{B}^T(\xi)&-rP
\end{bmatrix} \prec	 \gamma I \label{eq:opt_problem2_2} \\ 
&
& & 
\begin{bmatrix}
s\theta(X(\xi)+X(\xi)^T) & c\theta(X(\xi)-X(\xi)^T)\\
c\theta(X(\xi)^T-X(\xi)) & s\theta(X(\xi)+X(\xi)^T)
\end{bmatrix} \prec	 \gamma I \label{eq:opt_problem2_3} \\ 
&
& & 
-P  \prec	 \gamma I\\
&
& & 
\gamma I \prec 0
\end{align}
\end{small}


for all $\xi \in \Xi$, with $X(\xi)=\tilde{A}(\xi)P-\tilde{B}(\xi)Y$ and $s\theta = sin(\theta), c\theta = cos(\theta)$.

The controller is synthesized based on the solution of this optimization problem with $Y=L_cP$ and $L_c$ is obtained by $L_c=YP^{-1}$. The three constraints define the $\mathscr{D}$-region with the parameters $\alpha$, $r$, and $\theta$ for respectively stability and minimal dynamics, damping ratio, and maximal dynamics requirements. Note that other optimization objectives can be chosen.

The optimization problem (\ref{eq:opt_problem2}) depends on uncertain parameters that come from geometric uncertainties on contact points and operating points on the object and the joint spaces. Thus, it can be considered as a robust convex program (RCP):

\begin{equation}
\label{eq:RCP}
\text{RCP}:\begin{cases}\begin{aligned}
& \underset{P, \; Y , \; \gamma}{\text{min}}
& & \gamma \\
& \text{s.t.}
& & f(P,Y, \gamma, \xi) \leq 0, \; \xi \in \Xi
\end{aligned}
\end{cases}
\end{equation}
where $f(P,Y, \gamma, \xi) \leq 0$ is the set of LMIs given in (\ref{eq:opt_problem2}) for all possible $\xi \in \Xi$.  One way to tackle this problem is to consider all possible realizations of $\xi$ and seek a solution. However, since there are mostly infinite possible number of uncertainties, the optimization problem is intractable. The other way to handle an uncertain optimization problem is to take random samples from uncertain parameters affecting the problem and solve the convex program based on these random samples. This approach is called \textit{scenario approach} and the related problem is called Scenario Convex Program (SCP) \cite{calafiore2006scenario}. SCPs aim to find an optimal solution subject to a finite number of realizations of a constraint function with samples of uncertain parameters, namely scenarios. Thus, the RCP in (\ref{eq:RCP}) turns into the SCP by considering $N$ independent and identically distributed (i.i.d.) samples $(\xi_i)_{i=1}^N$ drawn according to a probability measure $\mathbb{P}$. The RCP in (\ref{eq:RCP}) turns into the SCP
\begin{equation}
\label{eq:SCP}
\text{SCP}:\begin{cases}\begin{aligned}
&  \underset{P, \; Y , \; \gamma}{\text{min}}
& & \gamma \\
& \text{s.t.}
& & f(P,Y, \gamma, \xi_i) \leq 0, \; i=1,\dots,N.
\end{aligned}
\end{cases}
\end{equation}
The natural question is how many samples of uncertain parameters is sufficient to find an acceptable solution. The answer to this question is addressed from two different points of view: feasibility of a solution and the optimal value of the objective function as detailed in the following subsections. 

\subsection{{Feasible solution}}
The result for the feasibility of a solution is given below with probabilistic guarantees.
\begin{theorem}\label{thm:1}
Let $\varepsilon \in (0,1)$ be a level parameter, $\beta \in (0,1)$ be a confidence parameter. Let $(\hat{P}, \hat{Y} , \hat{\gamma})$ be the solution to the SCP in (\ref{eq:SCP}). Then, with confidence $1-\beta$, we have
\begin{equation}
\label{eq:guarantee}
    \mathbb{P}[\delta \in \Delta : f(P,Y,\gamma, \xi)>0]\leq \varepsilon ,
\end{equation}
if the number of samples $N$ is chosen as
\begin{equation}
\label{eq:bound2}
N(\varepsilon,\beta):=\text{min} \Biggl\{N\in \mathbb{N}\; \bigg| \; \displaystyle\sum_{i=0}^{n_{P,Y, \gamma}-1} \binom{N}{i}\varepsilon^i(1-\varepsilon)^{N-i} \leq \beta \Biggl \}
\end{equation}
where $n_{P,Y, \gamma}=\frac{n_o(n_o+1)}{2}+2n_o^2$ is the number of decision variables.
\end{theorem}

\begin{proof}
For each $i\in \{1,\ldots,N\}$, we have that SCP
\begin{equation}
\begin{cases}\begin{aligned}
&  \underset{P, \; Y , \; \gamma}{\text{min}}
& & \gamma \\
& \text{s.t.}
& & f(P,Y,\gamma, \xi_i) \leq 0.
\end{aligned}
\end{cases}
\end{equation}
is either unfeasible or, if feasible, attains a unique optimal solution. Hence, the results hold by application of \cite[Theorem 1]{calafiore2006scenario}.
\end{proof}

The given result presents bounds on the number of data to ensure some guarantees on the constraint violation probability, that is, guarantees on the feasibility of an optimal solution. The connection between an RCP's and an SCP's optimal values is investigated in the next section. 

\subsection{{Optimal solution}}

In the following, we provide probabilistic guarantees on the mismatch between the optimal solutions of the RCP and a strengthened version of the SCP and give a sufficient number of samples to ensure the desired precision on this mismatch by following the approach in \cite{esfahani2014performance}. Indeed, while the result in Theorem \ref{thm:1} does not require any constraints on the decision variables $P$ and $Y$, the approach in \cite{esfahani2014performance} imposes additional constraints on the decision variables $P$ and $Y$ to measure the mismatch between the optimal solutions of the RCP and the SCP. 

Let us first define a constrained version of the RCP in (\ref{eq:RCP}) as follows:
\begin{equation}
\label{eq:RCPgamma}
\text{RCP}:\begin{cases}\begin{aligned}
& \underset{P, \; Y, \; \gamma}{\text{min}}
& & \gamma \\
& \text{s.t.}
& & O \leq P \leq \mu I \\
& 
& & 
-\mu I \leq Y \leq \mu I \\
& 
& & f(P,Y, \gamma, \xi) \leq 0, \; \xi \in \Xi 
\end{aligned}
\end{cases}
\end{equation}
Where $\mu >0$ is a design parameter. Now, consider $\alpha \geq 0$ and the following SCP:

\begin{equation}
\label{eq:SCPgamma}
\text{SCP}:\begin{cases}\begin{aligned}
&  \underset{P, \; Y , \; \gamma}{\text{min}}
& & \gamma \\
& \text{s.t.}
& & O \leq P \leq \mu I \\
& 
& & - \mu I \leq Y \leq \mu I \\
& 
& & f(P,Y, \gamma, \xi_i)+\alpha \leq 0, \; i=1,\dots,N.
\end{aligned}
\end{cases}
\end{equation}

We are now ready to state the result showing how the choose the number of samples $N$ and the parameter $\alpha$ in order to achieve a desired mismatch between the solutions of the RCP in (\ref{eq:RCPgamma}) and the SCP in (\ref{eq:SCPgamma}). Before presenting the result, we first have the following auxiliary lemmas, for which the proof can be found in the appendix.

\begin{lemma}\label{lemma:1}
    The constraint function $f(P,Y, \gamma, \xi_i)$ in (\ref{eq:RCP}) described by the maps $f_1(P,Y, \gamma, \xi)$ for (\ref{eq:opt_problem2_1}), $f_2(P,Y, \gamma, \xi_i)$ for (\ref{eq:opt_problem2_2}) and $f_3(P,Y, \gamma, \xi)$ for (\ref{eq:opt_problem2_3}) is Lipschitz continuous in $\xi$ uniformly in $P,\, Y ,\, \gamma$ with a Lipschitz constant $L=2\mu (L_A+L_B)\max(1,(|
\sin(\theta)|+|\cos(\theta)|))$, where $L_A$ and $L_B$ are the Lipschitz constants\footnote{The computation of the Lipschitz constants for the maps $\xi \mapsto \tilde{A}(\xi), \tilde{B}(\xi)$ can be done by resorting to existing tools in the literature \cite{darup2018fast,jerray2021orbitador}.} of the maps $\xi \mapsto \tilde{A}(\xi), \tilde{B}(\xi)$ defined in (\ref{eqn:dyn}).  
\end{lemma}

\begin{theorem}
\label{thm:SP}
Consider $\beta \in (0,1]$ and $\varepsilon \in [0,1]$ and assume the number of samples $N$ for the SCP in (\ref{eq:SCPgamma}) satisfies
\begin{equation}
\label{eqn:sample_size}
   N \geq N((\frac{\varepsilon}{L_\xi})^{n_{\xi}},\beta).
\end{equation}
where $N(.,.)$ is defined in (\ref{eq:bound2}) and $L_{\xi}$ is the Lipschitz constant of the map $(.,.,\xi) \mapsto f(P,Y,\gamma,\xi)$ with respect to $(P,Y,\gamma)$. If the parameter $\alpha$ of the SCP in (\ref{eq:SCPgamma}) is chosen as $\alpha=L_{\xi} \varepsilon^{\frac{1}{n_{\xi}}}$, then the optimal value $J^*$ for the optimization problem (\ref{eq:RCPgamma}), and $J^*_N$ the optimal value for the optimization problem (\ref{eq:SCPgamma}) satisfy with a confidence $1-\beta$, $$||J^*-J^*_N|| \leq \varepsilon.$$
\end{theorem}

\begin{proof}
We prove the result by using~\cite[Theorem 4.3]{esfahani2014performance}. We first use the result of Lemma \ref{lemma:1} to compute the Lipschitz constant $L_{\xi}$ for the map $(.,.,.,\xi) \mapsto f(P,Y,\gamma,\xi)$ with respect to $(P,Y,\gamma)$. Moreover, the map $(P,Y,\gamma,\xi) \mapsto f(P,Y,\gamma,\xi)$ is linear and then convex with respect to $(P,Y,\gamma)$. Since the distribution on the set $\Xi$ to generate the scenarios $\{\xi_1,\xi_2,\ldots,\xi_N\}$ is uniform, we choose the map $g(\varepsilon)=\varepsilon^{n_{\xi}}$.
The result follows then from an application of~\cite[Theorem 4.3]{esfahani2014performance},~\cite[Remark 3.9]{esfahani2014performance} and~\cite[Remark 3.5]{esfahani2014performance}.
\end{proof}

Intuitively, the result of Theorem \ref{thm:SP} states that if the number of samples is chosen according to (\ref{eq:bound2}), then with confidence $(1-\beta)$ the optimal value $J^*$ for the optimization problem (\ref{eq:RCPgamma}), and $J^*_N$ the optimal value for the optimization problem (\ref{eq:SCPgamma}) satisfy $||J^*-J^*_N|| \leq \varepsilon$.

\section{Numerical Results}
This section presents the simulation results of the proposed control scheme. The considered example system \cite{caldas2015object} is a planar hand with two fingers, 3 degrees of freedom each, manipulating a rectangular object with uncertain geometry (Fig.\ref{fig:translation}). It is consistent with assumptions A1 to A5 from Section~\ref{sec:model}. Although simplified compared to the full 3D case, this setup is sufficient to illustrate the control laws' performance. Simulations are done using Matlab 2022a and the optimization problem is formulated with the YALMIP toolbox using the SeDuMi 1.3 \cite{lofberg2004yalmip}. 

The control objective is the simultaneous translation and reorientation of the object in the plane, with significant amplitudes at the scale of the object/hand dimensions (translation of 40\unit{\milli\meter} in the x-direction, rotation of $\approx11\deg$). 

\subsection{Considered controllers}
To demonstrate the performance of the proposed method, we design three different controllers: $C^{\Delta}$, $C_{feas}^{\Xi}$ and $C_{opt}^{\Xi}$:

\begin{itemize}
    \item The controller $C^{\Delta}$ is designed by following the discretization approach of \cite{caldas2015object,caldas2022task} which considers the set of uncertainties on the grasp map. The set of uncertain systems (\ref{eq:grasp_uncertain}) is constructed based on the set $\Delta$ of uncertain parameters only. Then, the optimization problem (\ref{eq:opt_problem2}) is solved by taking 46 samples from a gridding of the uncertain set $\Delta$. 
    
    \item For the $C^{\Xi}$ scenario-based controllers proposed in this paper, the extended uncertainty set is considered to take into account different operating points. In this case, the set of uncertainties is constructed based on the extended uncertain set $\Xi$ that consists of uncertainties on the grasp map \emph{and} operating points. The optimization problem (\ref{eq:opt_problem2}) is solved by taking random samples from the uncertain set $\Xi$. However, in this case, we exploit Theorem 1 and 2 to determine the number of samples for ensuring probabilistic guarantees on the solution of the optimization problem from the feasibility (Theorem 1) and the optimality (Theorem 2) perspectives, leading to respectively two solutions $C_{feas}^{\Xi}$ and $C_{opt}^{\Xi}$.
\end{itemize}

\subsection{Considered uncertainties}
We consider the translation error (Figure \ref{fig:translation}) for contact point uncertainties. We assume that the direction of the contact force is known, but the location of the contact point is assumed to be uncertain, so for this case, the grasp map becomes:
\begin{equation}
G=\begin{pmatrix}
    0&-1&0&1\\
    1&0&-1&0\\
    r_0&0&r_0&\delta
\end{pmatrix}
\end{equation}
where $r_0=17.5$\unit{\milli\metre} is the length of the rectangular object and $\delta \in [-4\unit{\milli\metre},5\unit{\milli\metre}]$ represents the translation uncertainty of the second contact point.
\begin{figure}[htb]
\centering
\includegraphics[width=0.8\columnwidth]{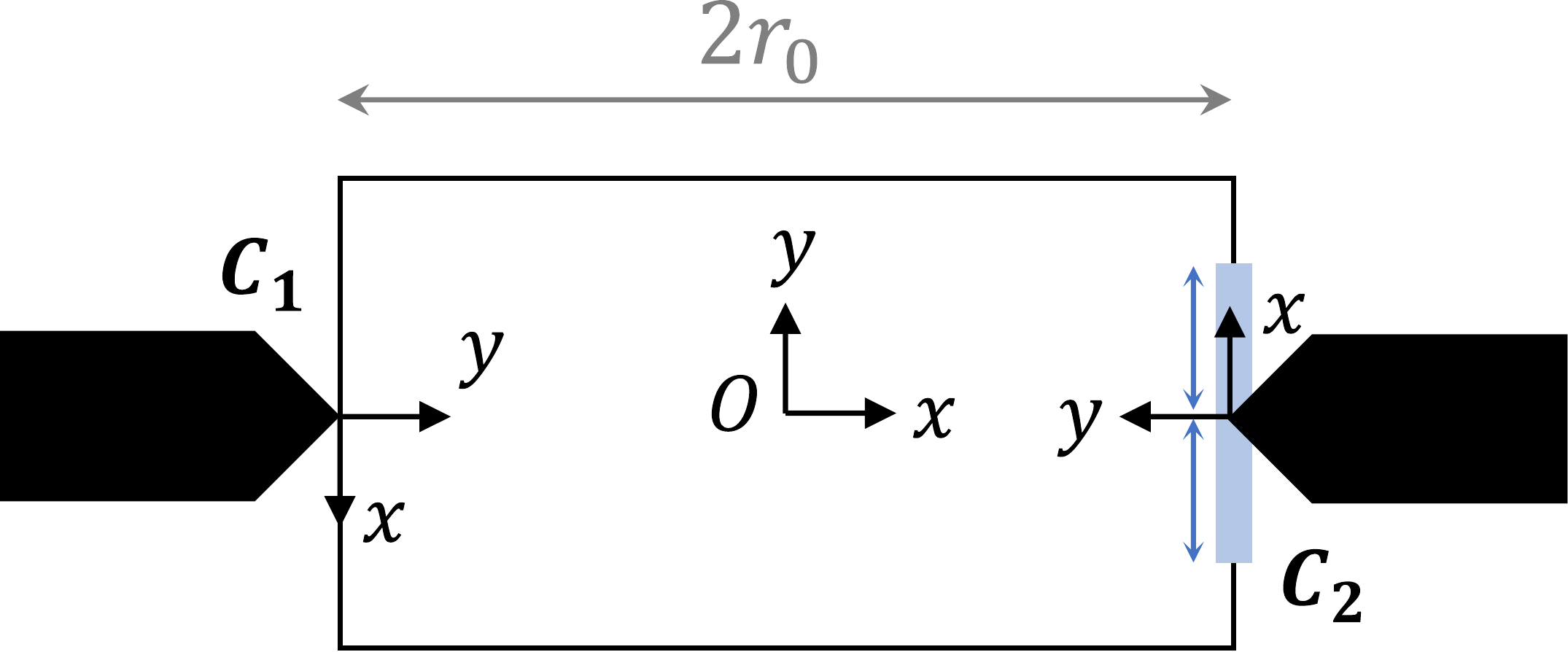}
\caption{Geometry of the manipulated planar object. Contact point C2's position is uncertain (blue region). Object's pose is locally parametrized by $x_o = [P_x\ P_y\ P_{\theta}]^T$.}
\label{fig:translation}
\end{figure}

The equilibrium point for the $C^{\Delta}$ discretized controller's design is chosen as $q_{eq}=\begin{bmatrix}0.9250&1.1170&0.9250&1.0647
\end{bmatrix}^T$\unit{\radian} and $x_{o_{eq}}=\begin{bmatrix}-17.5&66.5&0\end{bmatrix}^T$\unit{\milli\metre}. 

For the proposed $C^{\Xi}$ scenario-based controllers, we consider the same equilibrium point as before and we assume that the operating point for the joint angles varies in $q_{*}=[0.6632,0.9250]\times [1.1170,1.7453]\times [0.6632,0.9250]\times [1.0472,1.7453]$\unit{\radian} and the object position varies in $y$ coordinate within $[36.5,66.5]$\unit{\milli\metre}.

\subsection{Control design}
The motion control is designed to cluster the closed-loop poles in the $\mathscr{D}$-region, which is shaped by the following parameters:
\begin{itemize}
    \item $\alpha=0.5$ which ensures $Re(pole)<-0.5$, providing a stability constraint and minimal dynamics.
    \item $\theta=30$, which guarantees a minimum damping ratio $\zeta \simeq 0.86$.
    \item $r = 7$, which sets the maximal dynamics.
\end{itemize}

\vspace{5pt}
\noindent \textbf{Design of $C_{feas}^{\Xi}$ :} The number of samples that are required to ensure probabilistic guarantees on the feasibility of the solution of SCP are obtained through Theorem \ref{thm:1}. We fix the level parameter $\epsilon=0.5$ and the confidence $\beta=10^{-3}$, then using the bound (\ref{eq:bound2}), for the number of decision variables $n_{P,Y, \gamma}=39$, the required sample size to ensure the probabilistic guarantee in (\ref{eq:guarantee}) is obtained as $N=111$. Then, we solve the optimization problems (\ref{eq:opt_problem2}) by taking $N$ samples from the extended uncertainty set $\Xi$. 

\noindent \textbf{Design of $C_{opt}^{\Xi}$ :} To guarantee that the difference between the optimal values of the optimization problem and the optimal value of the scenario optimization problem is below some $\varepsilon$ with probability $1-\beta$, we need to use Theorem \ref{thm:SP}. According to Theorem \ref{thm:SP}, we need to compute the Lipschitz constants of the constraints. Thus, the Lipschitz constants are obtained as $L_{\xi_1}=7.4713$, $L_{\xi_2}=8.0188$, $L_{\xi_3}=2.7833$ using Lemma \ref{lemma:1}. One drawback of this approach is that the bound on the required number of samples (\ref{eqn:sample_size}) results in a high number of samples to ensure guarantees on the optimality. Since the system is more sensitive to joint angle variations, we can restrict the uncertain set $\Xi$ by fixing $\tau_{eq}$, the grasp map parameter $\delta$, and the object position $x_{o_{eq}}$. This way, the number of uncertain parameters to be sampled is reduced to $n_{\xi}=4$. We also fix the level parameter $\epsilon=0.99$ and the confidence $\beta=0.999$. Thus, the number of samples required for the optimality guarantees is obtained as $N=111714$. Due to Matlab limitations, only $N=16607$ samples have been used. Let us note that there is a tradeoff between the choice of the parameters $\beta$ and $\varepsilon$ and the computational complexity of the problem. Specifically, smaller values of  $\beta$ and $\varepsilon$ require a larger number of samples $N$, which increases the computational demands of solving the SCP problem in (\ref{eq:SCPgamma}). In practice, one should begin by considering the available computational resources, then determine the number of samples $N$ that can be feasibly solved, and subsequently identify the achievable values of the parameters $\beta$ and $\varepsilon$.

\subsection{Results}

The designed controllers are evaluated in closed-loop simulation with the non-linear hand-object system (\ref{eq:NLss}) with (\ref{eq:NLtransform}) and (\ref{eq:ctrl}). The performance of the controllers is tested for trajectories starting at two different initial points with zero velocities (IC stands for "initial conditions"):
\begin{itemize}
    \item IC1 : $q^1_0=\begin{bmatrix}0.9250&1.1170&0.9250&1.0647\end{bmatrix}^T$ and $x^1_{o_0}=\begin{bmatrix}-17.5&66.5&0\end{bmatrix}^T$, which corresponds exactly to the $C^{\Delta}$-design equilibrium point $q_{eq}$ and $x_{o_{eq}}$.
    \item IC2 : $q^2_0=\begin{bmatrix}0.6807&1.7453&0.6981&1.7453\end{bmatrix}^T$ and $x^2_{o_0}=\begin{bmatrix}-17.5&36.5&0\end{bmatrix}^T$, which belongs to the set of operating points in the $C^{\Xi}$-design.
\end{itemize}

Figures~\ref{fig:positions_trace_ic1_new} and \ref{fig:positions_trace_ic2_new} illustrate the closed-loop tracking performance for respectively the IC1 and IC2 trajectories. For the IC1 condition, all three controllers ensure reasonable tracking performance. For the IC2 condition, the discretized $C^{\Delta}$ fails to ensure stable motion, while the $C^{\Xi}_{opt}$ leads to the most performant tracking in terms of tracking error. The corresponding object motion is represented in Fig.~\ref{fig:strob}, where contact forces within the friction cones show that the contacts are maintained through the motion. The corresponding control torques, not shown here, also remain within physically meaningful intervals.
\vspace{-10pt}
\begin{figure}[htb]
\centering
\includegraphics[height=1.05\columnwidth, angle=-90]{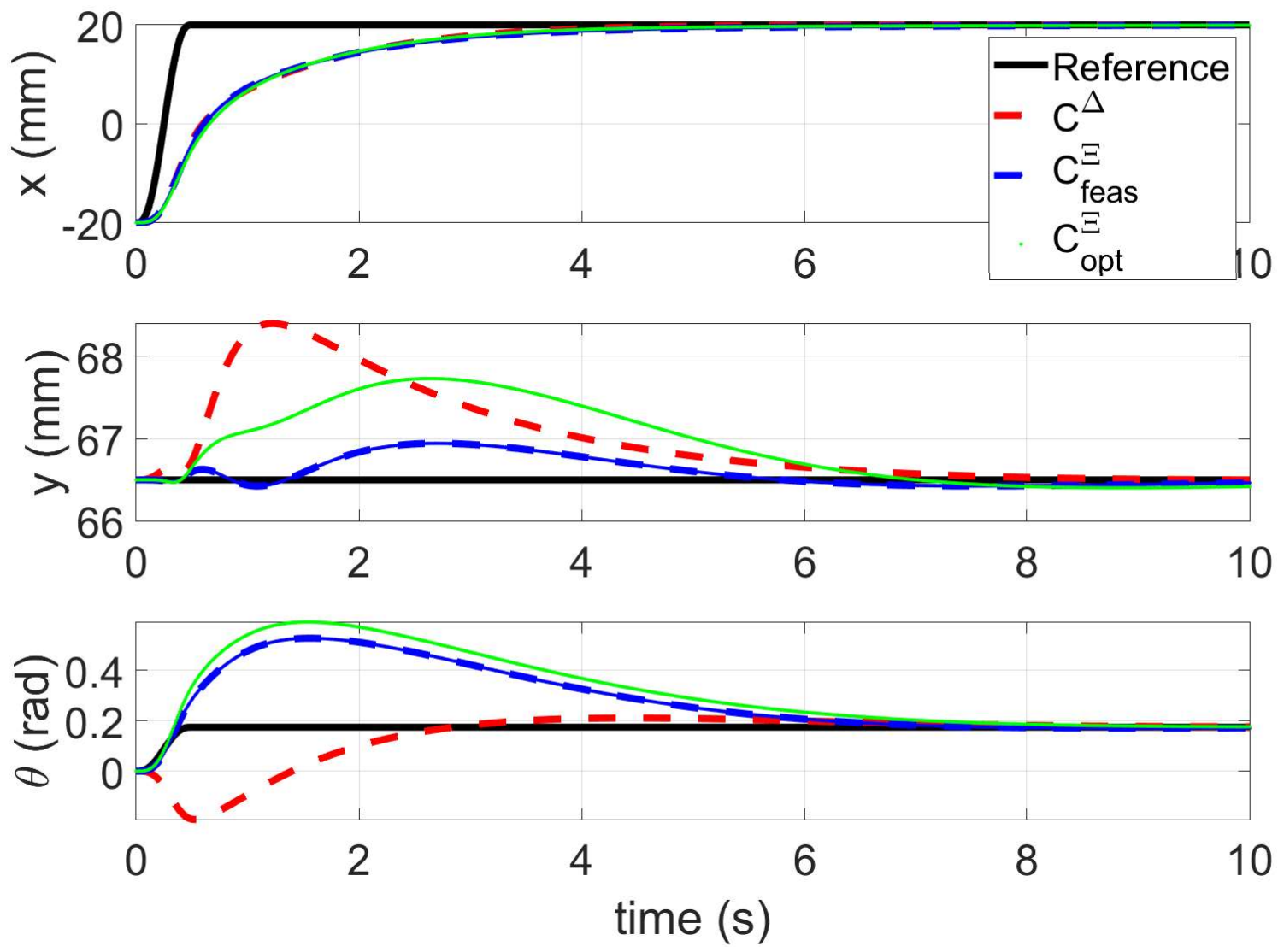}
\caption{IC1 Motion tracking performance of $C^{\Xi}_{feas}$ controller (blue dashed line), $C^{\Xi}_{opt}$ controller (green solid line), compared to the $C^{\Delta}$ controller (red dashed line) for the initial conditions $q^1_0$ and $x^1_{o_0}$ and the reference signals (black solid line).}
\label{fig:positions_trace_ic1_new}
\end{figure} 
\vspace{-10pt}
\begin{figure}[htb]
\centering
\includegraphics[height=1.05\columnwidth, angle=-90]{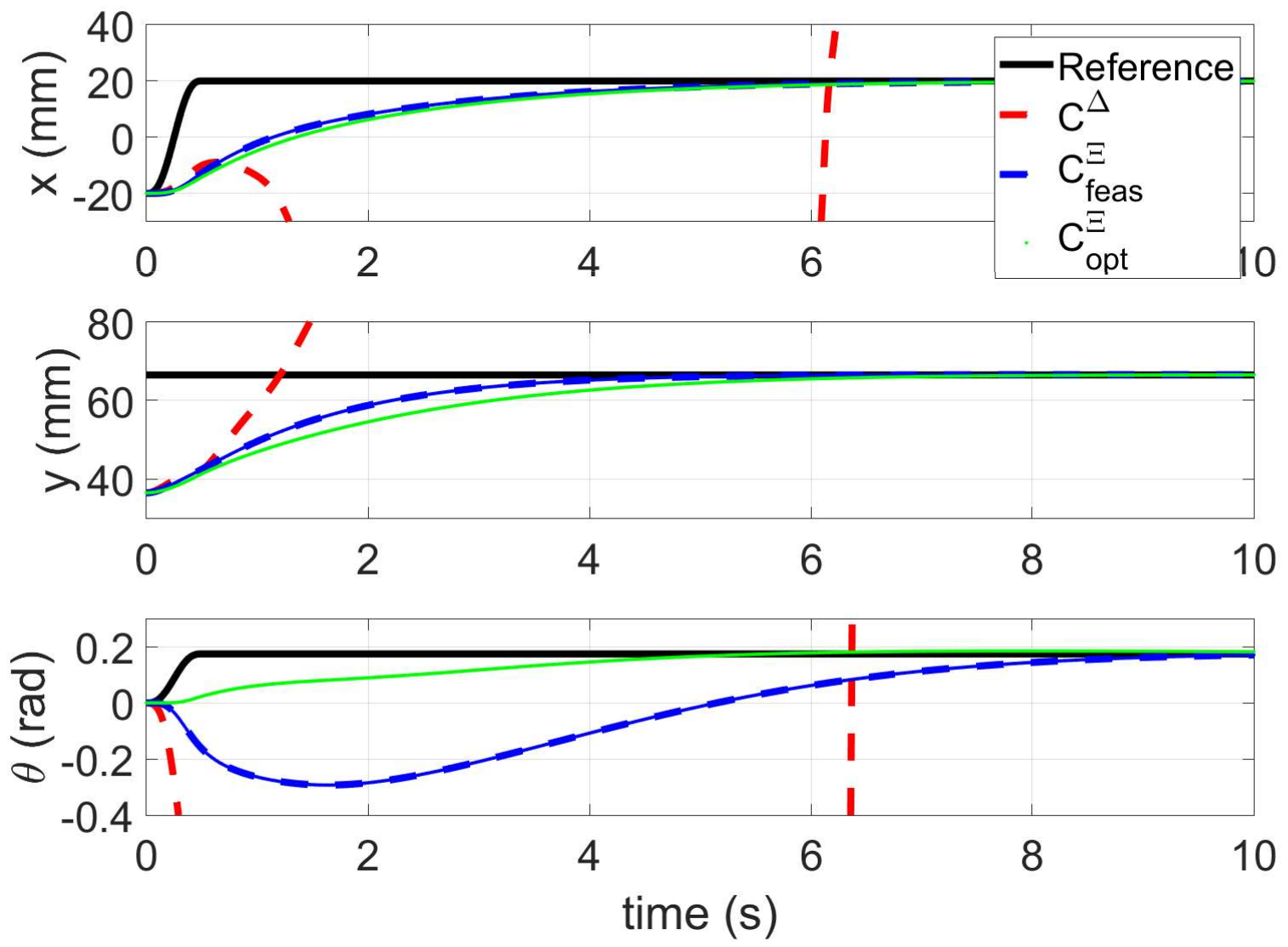}
\caption{IC2 Motion tracking performance of $C^{\Xi}_{feas}$ controller (blue dashed line), $C^{\Xi}_{opt}$ controller (green solid line), compared to the $C^{\Delta}$ controller (red dashed line) for the initial conditions $q^2_0$ and $x^2_{o_0}$ and the reference signals (black solid line).}
\label{fig:positions_trace_ic2_new}
\end{figure} 

\begin{figure}     
         \includegraphics[width=1\columnwidth]{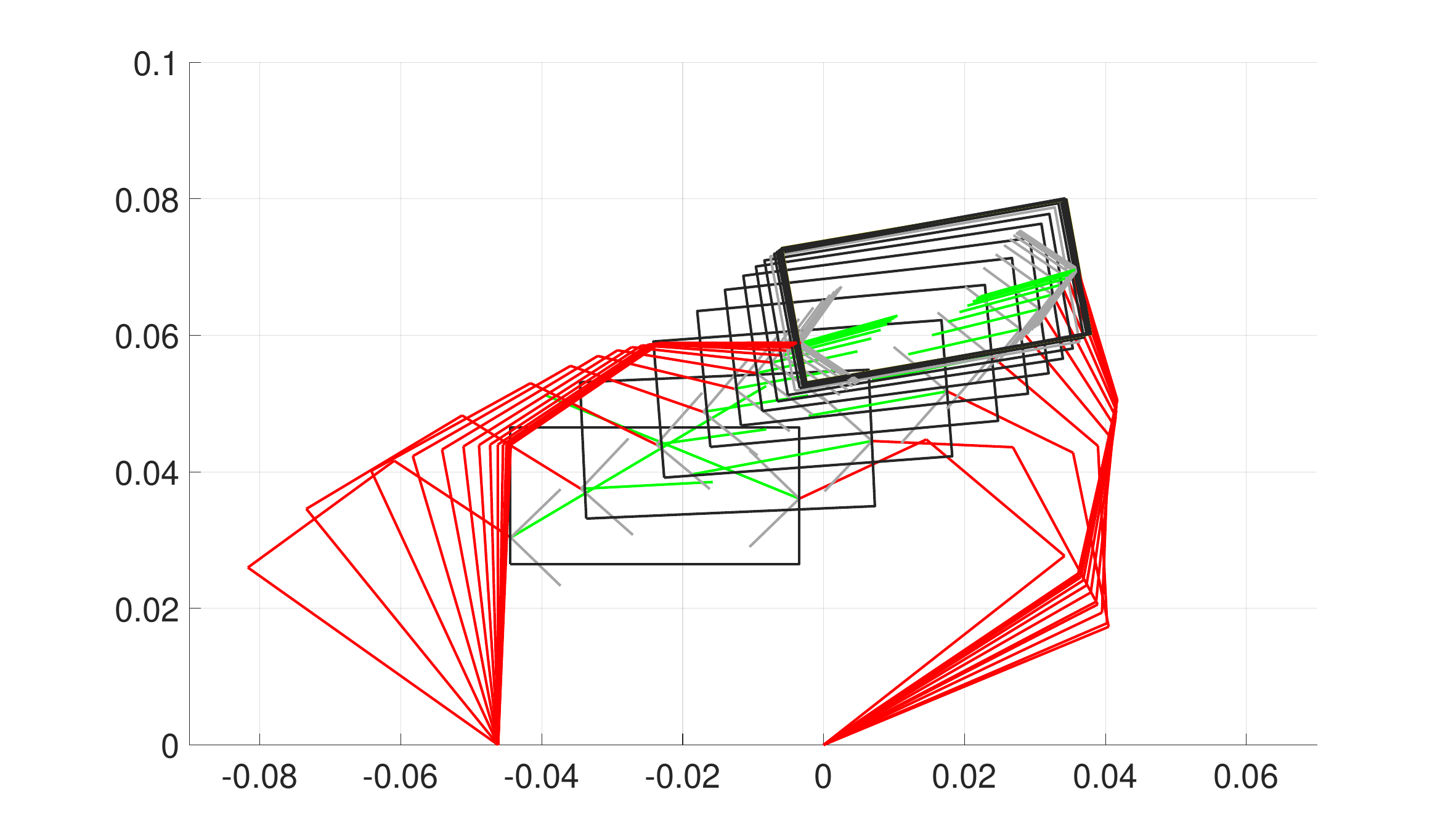}
         \caption{IC2 Stroboscopic view of the system response with the $C^{\Xi}_{opt}$ controller for the initial conditions $q^2_0$ and $x^2_{o_0}$ (corresponding to the green curves in Fig. \ref{fig:positions_trace_ic2_new}. A rectangular object (black lines) is manipulated by two planar three-degree-of-freedom fingers (red lines). Contact forces (green lines) remain inside the friction cones (grey lines). }
        \label{fig:strob}
\end{figure}

Figures~\ref{fig:poles_ic1} and \ref{fig:poles_ic2} further illustrate the closed-loop poles locations along the trajectories in respectively the IC1 and IC2 cases, for the three compared controllers. The closed-loop poles with the optimal scenario controller $C^{\Xi}_{opt}$ display the smallest dispersion for different operating points, thus indicating improved robustness with respect to previous designs. 
In all three cases, it can be observed that LMI constraints are not fulfilled; this is attributed to the fact that the poles are computed based on operating points not considered during the controller design.

\section{CONCLUSIONS}
In this paper, we propose a scenario-based controller design that considers different operating points to overcome the limitations of linear models and takes into account contact point uncertainties to ensure robustness for multi-fingered robot systems. The regional pole placement in the LMI formulation is adopted to achieve this goal and it is formulated as a RCP. To handle the RCP, we reformulate it as a SCP and we rely on scenario optimization to find a required sample size for ensuring probabilistic guarantees on the solution. The required sample size is calculated from a feasibility and optimality perspective.

Simulation results show that when we construct our controller based on a feasible solution of regional pole placement problem, the closed-loop system is limited to operate around the linearization point when no further care is taken about possible operating points. The sensitivity of the closed-loop to the operating point variations can lead to unstable behavior and grasp loss. The RCP formulation of the problem leads to a solution with improved robustness. In future work, we aim at an experimental evaluation of the proposed method.

\begin{figure}[htb]
\centering
\includegraphics[width=1.05\columnwidth]{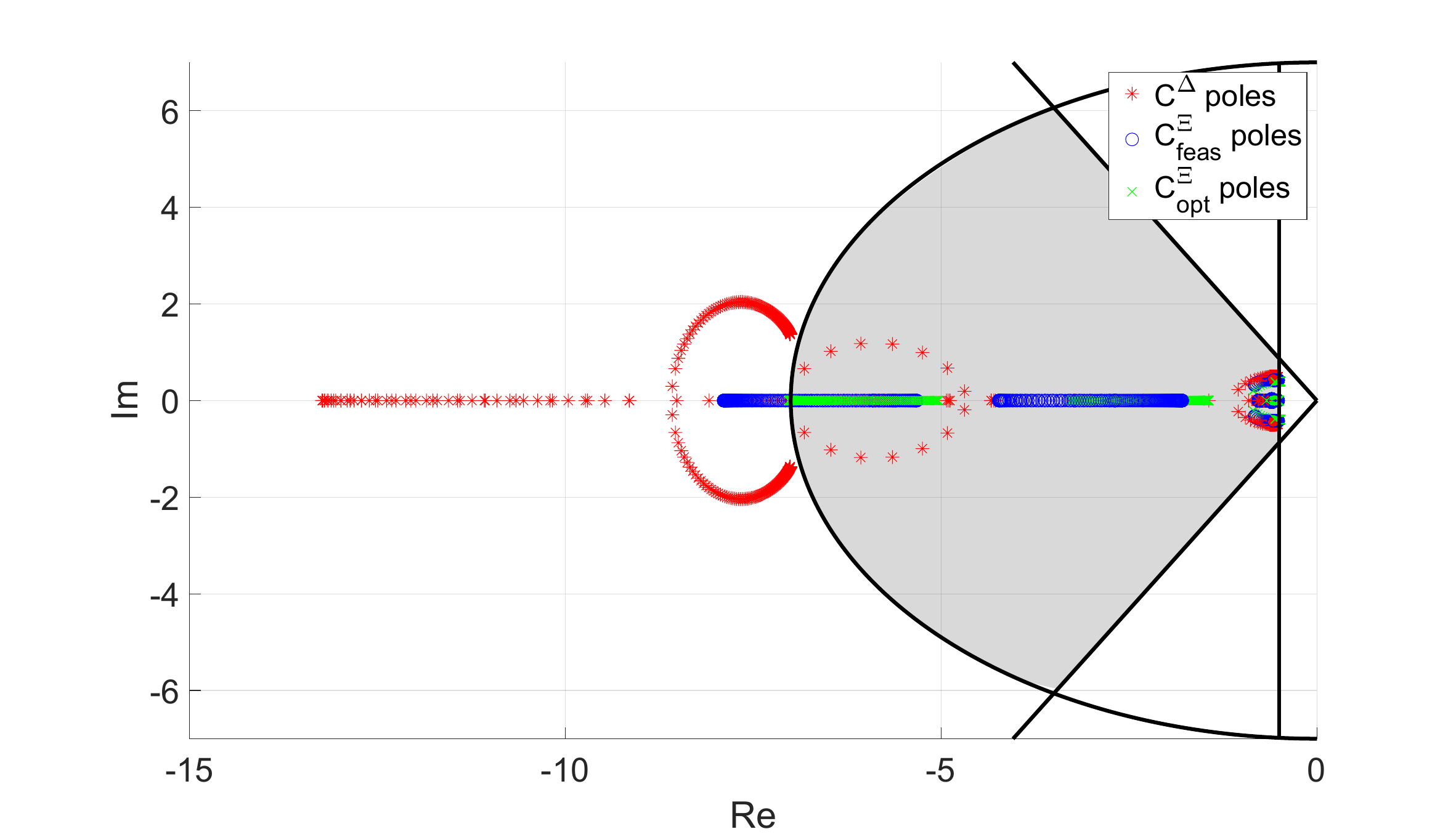}
\caption{Closed-loop poles evaluated on the IC1 test trajectory. With the $C^\Xi_{opt}$ controller design (green), the closed-loop poles remain within the desired D-stability region (grey area).}
\label{fig:poles_ic1}
\end{figure} 

\begin{figure}[htb]
\centering
\includegraphics[width=1.05\columnwidth]{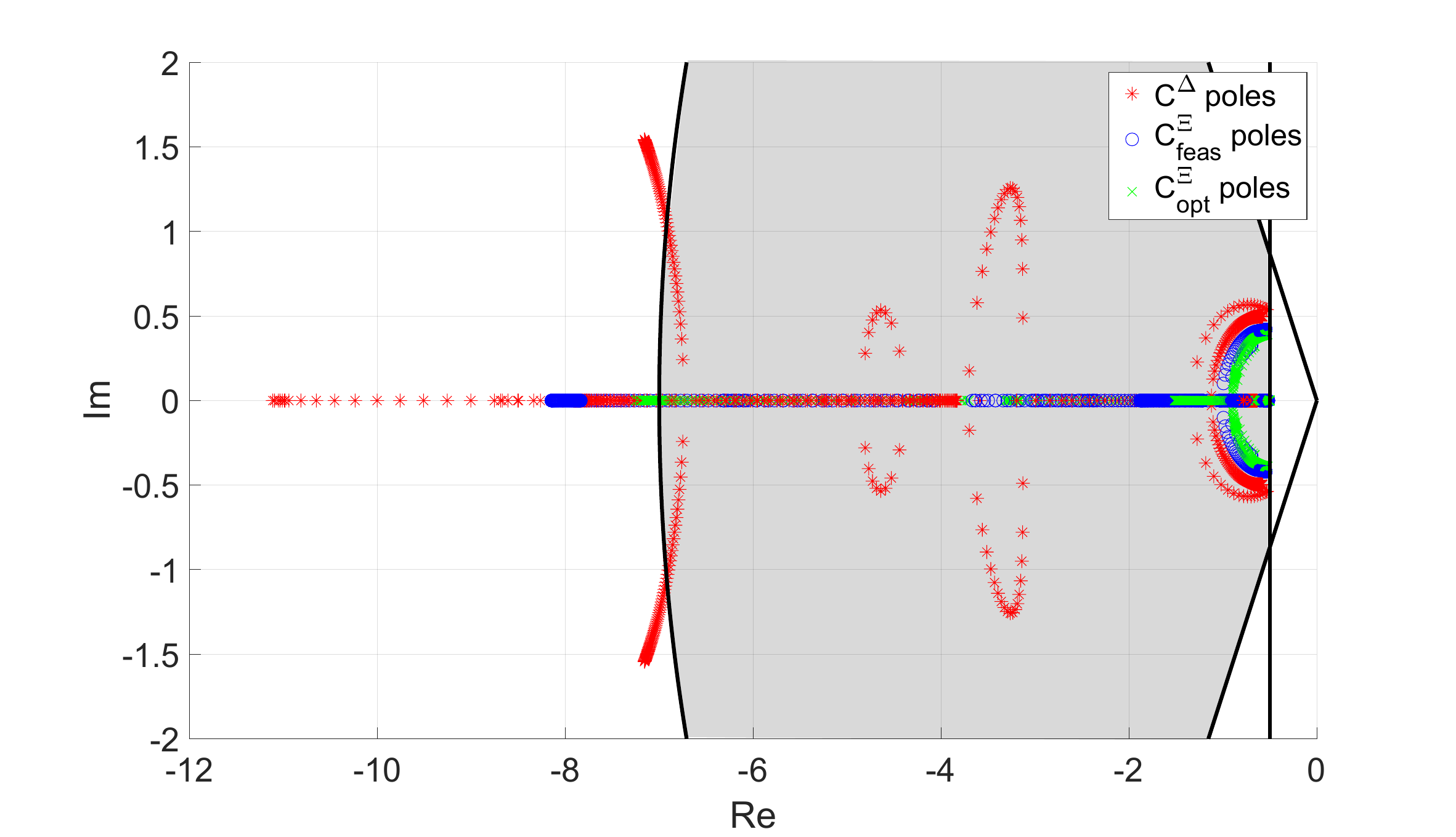}
\caption{Closed-loop poles  evaluated on the IC2 test trajectory. With the $C^\Xi_{opt}$ controller design (green), the closed-loop poles remain within the desired D-stability region (grey area) and present the smallest dispersion for different operating points, indicating improved robustness.}
\label{fig:poles_ic2}
\end{figure}





\section*{APPENDIX}
\label{sec:appendix}





\begin{proof}
    We prove that constraints in (\ref{eq:opt_problem2}) are Lipschitz continuous in $\xi$ uniformly in $P,\, Y ,\, \gamma$. Let us first deal with the first constraint in (\ref{eq:opt_problem2_1}). Consider $\xi_1,\xi_2 \in \Xi$ we have
\begin{align*}
||&\tilde{A}(\xi_1)P-\tilde{B}(\xi_1)Y+P\tilde{A}^T(\xi_1)-Y^T\tilde{B}^T(\xi_1)+2\alpha P\notag\\
&-(\tilde{A}(\xi_2)P-\tilde{B}(\xi_2)Y+P\tilde{A}^T(\xi_2)-Y^T\tilde{B}^T(\xi_2)+2\alpha P)||\notag\\
=&||(\tilde{A}(\xi_1)-\tilde{A}(\xi_2))P-(\tilde{B}(\xi_1)-\tilde{B}(\xi_2))Y\\&+P(\tilde{A}^T(\xi_1)-\tilde{A}^T(\xi_2))-Y^T(\tilde{B}^T(\xi_1)-\tilde{B}^T(\xi_2))||\\
\leq&||(\tilde{A}(\xi_1)-\tilde{A}(\xi_2))+(\tilde{A}^T(\xi_1)-\tilde{A}^T(\xi_2)||.||P||\\&+||(\tilde{B}(\xi_1)-\tilde{B}(\xi_2))+(\tilde{B}^T(\xi_1)-\tilde{B}^T(\xi_2)||.||Y||\\ 
\leq&2 \mu( ||\tilde{A}(\xi_1)-\tilde{A}(\xi_2)||+||\tilde{B}(\xi_1)-\tilde{B}(\xi_2)||)\\
\leq&2 \mu(L_A+L_B)( ||\xi_1-\xi_2||)
\end{align*}
and $L_1=2 \mu(L_A+L_B)$ is a Lipschitz constant of the first constraint in (\ref{eq:opt_problem2_1}).

For the second constraint in (\ref{eq:opt_problem2_2}), and for $\xi_1,\xi_2 \in \Xi$, one gets:
\begin{align*}
&\Bigg|\Bigg|\begin{bmatrix}
-rP&\tilde{A}(\xi_1)P-\tilde{B}(\xi_1)Y\\
P\tilde{A}^T(\xi_1)-Y^T\tilde{B}^T(\xi_1)&-rP
\end{bmatrix}\\& - \begin{bmatrix}
-rP&\tilde{A}(\xi_2)P-\tilde{B}(\xi_2)Y\\
P\tilde{A}^T(\xi_2)-Y^T\tilde{B}^T(\xi_2)&-rP
\end{bmatrix}\,\Bigg|\Bigg|\\  &\Bigg|\Bigg|\left[\begin{matrix}
0\\
P(\tilde{A}^T(\xi_1)-\tilde{A}^T(\xi_2))-Y^T(\tilde{B}^T(\xi_1)-\tilde{B}^T(\xi_2))
\end{matrix} \right. \\ &\qquad\qquad\qquad\left. \begin{matrix}
(\tilde{A}(\xi_1)-\tilde{A}(\xi_2))P-(\tilde{B}(\xi_1)-\tilde{B}(\xi_2))Y\\0
\end{matrix}\right] \Bigg|\Bigg| \\ =& ||(\tilde{A}(\xi_1)-\tilde{A}(\xi_2))P-(\tilde{B}(\xi_1)-\tilde{B}(\xi_2))Y\|| \\
\leq&\mu(L_A+L_B)( ||\xi_1-\xi_2||)
\end{align*}
and $L_2=\mu(L_A+L_B)$ is a Lipschitz constant of the first constraint in (\ref{eq:opt_problem2_2}).

Consider the third constraint in (\ref{eq:opt_problem2_3})
and $\xi_1,\xi_2 \in \Xi$. Using the derivations in (\ref{eqn:lipschitz}) one gets that $L_3=2\mu (L_A+L_B)(|
\sin(\theta)|+|\cos(\theta)|)$ is a Lipschitz constant of the constraint in (\ref{eq:opt_problem2_3}).

\begin{figure*}[t!]
\rule{\textwidth}{0.4pt}
\begin{small}
\begin{align}
\label{eqn:lipschitz}
& \Bigg|\Bigg|\, \begin{bmatrix}
sin(\theta)(X(\xi_1)+X^T(\xi_1))&cos(\theta)(X(\xi_1)-X^T(\xi_1)) \nonumber\\
cos(\theta)(X^T(\xi_1)-X(\xi_1))&sin(\theta)(X(\xi_1)+X^T(\xi_1))
\end{bmatrix}-\begin{bmatrix}
sin(\theta)(X(\xi_2)+X^T(\xi_2))&cos(\theta)(X(\xi_2)-X^T(\xi_2)) \nonumber\\
cos(\theta)(X^T(\xi_2)-X(\xi_2))&sin(\theta)(X(\xi_2)+X^T(\xi_2))
\end{bmatrix} \,\Bigg|\Bigg| \nonumber\\&= \Bigg|\Bigg|\, \begin{bmatrix}
sin(\theta)((X(\xi_1)-X(\xi_2))+(X^T(\xi_1))-X^T(\xi_2))&cos(\theta)((X(\xi_1)-X(\xi_2)-(X^T(\xi_1)-X^T(\xi_2))) \nonumber\\
cos(\theta)((X^T(\xi_1)-X^T(\xi_2))-(X(\xi_1)-X(\xi_2)))&sin(\theta)((X(\xi_1)-X(\xi_2))+(X^T(\xi_2)-X^T(\xi_2)))
\end{bmatrix} \nonumber\\ &\leq \Bigg|\Bigg|\, \begin{bmatrix}
sin(\theta)((X(\xi_1)-X(\xi_2))+(X^T(\xi_1))-X^T(\xi_2))&0\\0&sin(\theta)((X(\xi_1)-X(\xi_2))+(X^T(\xi_2)-X^T(\xi_2)))
\end{bmatrix} \nonumber\\ &+ \Bigg|\Bigg|\, \begin{bmatrix}
0&cos(\theta)((X(\xi_1)-X(\xi_2)-(X^T(\xi_1)-X^T(\xi_2))) \nonumber\\
cos(\theta)((X^T(\xi_1)-X^T(\xi_2))-(X(\xi_1)-X(\xi_2)))&0
\end{bmatrix} \nonumber\\ &\leq ||
sin(\theta)((X(\xi_1)-X(\xi_2))+(X^T(\xi_1))-X^T(\xi_2))||+ ||cos(\theta)((X(\xi_1)-X(\xi_2)-(X^T(\xi_1)-X^T(\xi_2)))|| \nonumber\\ &\leq 2(|
\sin(\theta)|+|\cos(\theta)|).||X(\xi_1)-X(\xi_2))|| \nonumber \\ &\leq 2(|
\sin(\theta)|+|\cos(\theta)|).||(\tilde{A}(\xi_1)P-\tilde{B}(\xi_1)Y)-(\tilde{A}(\xi_2)P-\tilde{B}(\xi_2)Y)|| \nonumber \\ &\leq 2\mu (L_A+L_B)(|
\sin(\theta)|+|\cos(\theta)|)( ||\xi_1-\xi_2||)
\end{align}
\end{small}
\rule{\textwidth}{0.4pt}
\end{figure*}

\end{proof}



\bibliographystyle{IEEEtran}
\bibliography{SCP_conf_references}

\end{document}